\documentclass{article}
\usepackage[preprint]{neurips_2026}
\usepackage[utf8]{inputenc}
\usepackage[T1]{fontenc}
\usepackage{hyperref}
\usepackage{url}
\usepackage{booktabs}
\usepackage{amsfonts}
\usepackage{amsmath,amssymb,amsthm}
\usepackage{nicefrac}
\usepackage{microtype}
\usepackage{xcolor}
\usepackage{algorithm}
\usepackage{algorithmic}
\usepackage{enumitem}
\usepackage{graphicx}
\usepackage{subcaption}
\usepackage{multirow}
\usepackage{tikz}
\usetikzlibrary{arrows.meta, positioning, fit, backgrounds}
\usepackage{placeins}
\usepackage{float}
\graphicspath{{figures/}}

\newtheorem{theorem}{Theorem}

\newtheorem{proposition}[theorem]{Proposition}
\newtheorem{corollary}[theorem]{Corollary}
\newtheorem{definition}{Definition}
\newtheorem{assumption}{Assumption}
\newtheorem{remark}{Remark}
\newcommand{\Estate}{E}
\newcommand{\Wstate}{W}
\newcommand{\spec}{\mathcal{S}}
\newcommand{\adapt}{\mathrm{AdaptCost}}
\newcommand{\comp}{\mathrm{comp}}

\title{Revisable by Design: A Theory of Streaming LLM Agent Execution\thanks{Code, data, and benchmark: \url{https://github.com/zhiyuanZhai20/stream-agent}.}}
\author{%
  Zhiyuan Zhai\\
  Fudan University\\
  \texttt{22110720067@m.fudan.edu.cn}
  \And
  Ming Li\thanks{Co-corresponding authors.}\\
  Guangming Lab\\
  \texttt{ming.li@u.nus.edu}
  \And
  Xin Wang\footnotemark[\value{footnote}]\\
  Fudan University\\
  \texttt{xwang11@fudan.edu.cn}
}

\begin{document}
\maketitle

\begin{abstract}
Current LLM agents operate under an implicit but universal assumption: execution is a \emph{transaction}---the user submits a request, the agent works in isolation, and only upon completion does the dialogue resume. This forces users into a binary choice: wait for a potentially incorrect output, or interrupt and lose all progress.

We reject this assumption and propose the \textbf{stream paradigm}, in which agent execution and user intervention are concurrent, interleaved processes sharing a bidirectional channel. We formalize this paradigm through a \emph{reversibility taxonomy} that classifies every agent action as Idempotent, Reversible, Compensable, or Irreversible, and arrive at a core conclusion: \textbf{an agent's flexibility is bounded by its reversibility}. We prove that conflicting compensable actions impose unavoidable adaptation costs and that conflicting irreversible actions make full specification satisfaction impossible---these costs are properties of the action space, not of the algorithm. Guided by this insight, we present the \textbf{Revision Absorber}, a reactive algorithm based on the \emph{Earliest-Conflict Rollback} rule that is structurally optimal under mild assumptions. Experiments on StreamBench with real LLM agents validate all predictions: the Absorber matches the quality of a brute-force full-restart baseline while wasting an order of magnitude fewer steps of already-completed work, turning mid-execution revisions from a dead-end into a first-class interaction.
\end{abstract}

\section{Introduction}
\label{sec:intro}

Large language model (LLM) agents---systems that interleave reasoning, tool use, and environmental observation to accomplish complex tasks~\citep{yao2023react,shinn2023reflexion}---have become a central paradigm in applied AI. Systems like Claude Code, OpenAI Codex agents, and Devin-style coding assistants routinely execute multi-step workflows spanning minutes to hours.

Throughout this paper, we use the term \emph{agent} to refer specifically to tool-using LLM systems that execute multi-step plans with observable side effects (e.g., file writes, API calls, message sends), as distinct from single-turn chat or pure reasoning models.

A striking property of all current agent architectures is an unstated assumption: \emph{execution is a transaction}. The user submits a query, the agent enters a closed execution loop (plan $\to$ act $\to$ observe $\to$ re-plan), and only upon completion---or failure---does the user regain control. We call this the \textbf{transactional assumption}.

The transactional assumption creates a fundamental tension. As agents tackle longer-horizon tasks, the probability that the user's intent \emph{evolves} during execution increases. The user watches intermediate outputs stream by and realizes: the direction is wrong, a constraint was omitted, priorities have shifted. Under the transactional model, the user faces a binary: wait for a potentially incorrect result, or interrupt and discard all progress. We propose a third path: an \emph{open} user--agent channel in which revisions arrive during execution and are absorbed on-the-fly (visual summary in Fig.~\ref{fig:overview}, after the taxonomy is formally defined).

\paragraph{This paper.} We propose replacing the transactional assumption with the \textbf{stream paradigm}: agent execution and user intervention coexist as two concurrent processes on a shared bidirectional channel. The agent continuously streams its execution to the user; the user may inject revisions at any time without blocking the agent. By formalizing this paradigm and analyzing what determines an agent's capacity to absorb revisions, we arrive at a core conclusion:

\begin{quote}
\emph{An agent's flexibility is bounded by its reversibility.}
\end{quote}

No matter how sophisticated the planner, once an agent has executed actions with irreversible effects on the external world, the set of reachable final outcomes is permanently reduced. Revisions requiring outcomes outside this set incur compensation costs, wasted work, or outright failure. This is not an algorithmic limitation---it is a structural property of the action space itself.

\paragraph{Contributions.}
\begin{enumerate}[leftmargin=*,itemsep=2pt]
\item \textbf{The stream paradigm} (\S\ref{sec:stream}). We propose a new interaction model in which the agent executes autonomously while the user may inject revisions at any time through a non-blocking bidirectional channel. We formalize this as the \emph{agent execution stream}, a typed event sequence over four event kinds: action, thought, observation, and injection.

\item \textbf{Formalization of reversibility} (\S\ref{sec:stream}--\ref{sec:revision}). We decompose agent state into two layers---epistemic (always rollbackable) and world (partially rollbackable)---and introduce a \emph{reversibility taxonomy} $\{I, R, K, X\}$ that classifies every action by its effect on world state. We define the \emph{reversibility ratio} $\rho(T)$ of a task, the \emph{adaptation cost} of absorbing a revision, and a minimal taxonomy of revision types.

\item \textbf{Core theoretical result: flexibility is bounded by reversibility} (\S\ref{sec:theory}). We prove that conflicting $K$-class actions impose unavoidable compensation costs and conflicting $X$-class actions make full specification satisfaction impossible (Proposition~\ref{prop:irreversibility-cost}). These costs are properties of the action space, not of the algorithm.

\item \textbf{The Revision Absorber algorithm} (\S\ref{sec:algorithm}). Guided by the above insight, we present a reactive algorithm based on the \emph{Earliest-Conflict Rollback} rule. We prove this rule is structurally optimal under mild assumptions (Theorem~\ref{thm:structural-opt}). The algorithm requires no cost estimation, no search, and at most $O(m)$ LLM calls.

\item \textbf{Experimental validation} (\S\ref{sec:experiments}). On StreamBench with DeepSeek-V3 as the agent ($n{=}1{,}008$ runs), the Absorber achieves quality statistically indistinguishable from Full-Restart while wasting $14.6\times$ fewer steps, empirically confirming all theoretical predictions.
\end{enumerate}

\section{Related Work}
\label{sec:related}

\paragraph{LLM agent architectures.}
ReAct~\citep{yao2023react} established the interleaved thought--action--observation loop that underlies most modern agents. Reflexion~\citep{shinn2023reflexion} added a self-correction mechanism in which the agent critiques its own output after task completion and retries. Plan-and-Execute architectures~\citep{wang2023planandsolve} separate high-level planning from step-level execution, enabling dynamic replanning when a step fails. A common thread in all these architectures is that the user's role ends at $t=0$: the initial query fully specifies the task, and the agent operates autonomously until completion. None models the possibility that the user's intent may evolve \emph{during} execution, and none reasons about the reversibility of the actions the agent takes.

\paragraph{Human-in-the-loop agents.}
LangGraph's \texttt{interrupt()} mechanism~\citep{langgraph2024} enables synchronous, agent-initiated checkpoints where the agent pauses to request human approval. Collaborative Gym (Co-Gym)~\citep{shao2024cogym} provides an asynchronous tripartite framework for human--agent--environment interaction and demonstrates that collaborative agents outperform autonomous ones. CowPilot~\citep{huq2025cowpilot} supports human intervention in web navigation via a suggest-then-execute model. However, all of these works remain fundamentally \emph{turn-based}: the agent either pauses and waits for human input (LangGraph), or the human and agent alternate in a structured protocol (Co-Gym, CowPilot). None realizes a true bidirectional non-blocking channel in which the agent executes naturally without distortion while the user injects revisions at arbitrary times and the agent absorbs them on the fly. Furthermore, none provides theoretical guarantees on adaptation cost or formal characterization of what determines adaptability.

\paragraph{Asynchronous LLM serving and runtimes.}
High-throughput serving systems such as vLLM~\citep{kwon2023vllm} and SGLang~\citep{zheng2024sglang} target token-level latency and throughput for many \emph{independent} requests through paged attention and structured program execution. Agent runtimes---LangGraph~\citep{langgraph2024}, AutoGen~\citep{wu2023autogen}, CrewAI~\citep{crewai2024}, Google's ADK~\citep{google2024adk}---provide bidirectional streaming between agents and tools for inter-agent coordination, but the user--agent relationship remains transactional: none provides a mechanism for the user to inject mid-execution revisions that the agent absorbs on the fly.

\paragraph{Plan repair in classical AI.}
Plan repair~\citep{fox2006plan,nebel1995plan} modifies partially executed plans when execution diverges from expectations. HTN plan repair~\citep{holler2020htn} extends this to hierarchical task networks, and recent benchmarks~\citep{valmeekam2023planbench} evaluate whether LLMs can reason about plan change. The plan commitment property~\citep{babli2023commitment} minimizes disruption to other agents in shared environments. These works operate within a well-studied setting---\emph{dynamic environments with fixed goals}---where the agent reacts to environmental changes while pursuing an unchanging objective. Our work is orthogonal, not incremental: we introduce a \emph{new paradigm} in which the agent and the user share a persistent bidirectional channel, the agent executes autonomously, and the user may inject goal revisions at any time. This setting is neither a special case nor a generalization of plan repair; it addresses a distinct source of plan invalidation (user-initiated goal change rather than environment dynamics) and identifies a distinct binding constraint (action reversibility rather than plan structure).

\section{The Stream Paradigm}
\label{sec:stream}

\subsection{Agent Execution Streams}

\begin{definition}[Agent Execution Stream]
An \emph{agent execution stream} is a finite sequence $\sigma = (\varepsilon_1, \varepsilon_2, \ldots, \varepsilon_n)$ of typed events, where each event $\varepsilon_i$ belongs to one of four types:
\begin{itemize}[leftmargin=*,itemsep=1pt]
    \item $\mathtt{act}(a, t)$: the agent executed action $a$ at time $t$, affecting world state $W$.
    \item $\mathtt{thk}(r, t)$: the agent emitted reasoning $r$ at time $t$, affecting epistemic state $E$ only.
    \item $\mathtt{obs}(o, t)$: the environment returned observation $o$ at time $t$, providing information.
    \item $\mathtt{inj}(\phi, t)$: the user injected revision $\phi$ at time $t$, exogenous and unpredictable.
\end{itemize}
Events are ordered by time. The first three types are \emph{endogenous}---generated by the agent--environment loop. The fourth is \emph{exogenous}---arriving from outside the agent's control flow.
\end{definition}

This abstraction simultaneously captures three things that simpler models miss: (i) the agent's internal control flow ($\mathtt{thk} \to \mathtt{act} \to \mathtt{obs}$), (ii) the exogenous user channel ($\mathtt{inj}$), and (iii) the temporal interleaving between them. It is \emph{not} a relabeling of ReAct traces: ReAct assumes user input is fully committed at $t = 0$, whereas our stream treats user input as \emph{progressively revealed} through $\mathtt{inj}$ events whose timing is controlled by the user.

\subsection{State Decomposition}

At any time $t$, the agent's state is the pair $(\Estate_t, \Wstate_t)$:

\begin{definition}[State Pair]~
\begin{itemize}[leftmargin=*,itemsep=1pt]
    \item \textbf{Epistemic state} $\Estate_t$: the agent's context, beliefs, and accumulated information. \emph{Always rollbackable} via context truncation.
    \item \textbf{World state} $\Wstate_t$: changes to the external environment caused by the agent's actions. \emph{Partially rollbackable}, depending on the action.
\end{itemize}
\end{definition}

The two layers have fundamentally different rollback semantics. Epistemic state is data; it can always be truncated. World state is the critical layer: its rollbackability depends on the algebraic structure of the actions that produced it. (Tokens and wall-clock time consumed during execution are sunk costs that cannot be recovered; we do not model them as part of the agent's state since they play no role in rollback decisions.)

\subsection{The Reversibility Taxonomy}

\begin{definition}[Reversibility Classes]
\label{def:reversibility}
For each action $a$ in the agent's action space $\mathcal{A}$, we assign a \emph{reversibility class} based on the algebraic structure of its effect on world state $W$:
\begin{itemize}[leftmargin=*,itemsep=1pt]
    \item \textbf{Idempotent ($I$)}: These actions do not modify world state $W$. Formally, $\mathrm{exec}(a) \circ \mathrm{exec}(a) \equiv \mathrm{exec}(a)$. Since $W$ is unchanged, no rollback is ever needed. \emph{Examples}: read file, query database, fetch web page.
    \item \textbf{Reversible ($R$)}: $\exists\, a^{-1}$ such that $\mathrm{exec}(a^{-1}) \circ \mathrm{exec}(a) \equiv \mathrm{id}$. Effect is exactly undoable. \emph{Examples}: create file $\leftrightarrow$ delete file; set flag $\leftrightarrow$ unset flag.
    \item \textbf{Compensable ($K$)}: No exact inverse exists, but $\exists\, \comp(a)$ such that the composition leaves $W$ in a state acceptable under the new specification. \emph{Examples}: sent email $\to$ send correction; booked ticket $\to$ cancel with fee.
    \item \textbf{Irreversible ($X$)}: No compensation exists or its cost is prohibitive. \emph{Examples}: settled financial transaction; deleted unique data without backup.
\end{itemize}
\end{definition}

This taxonomy mirrors the \emph{saga pattern} in distributed database theory~\citep{garcia1987sagas} and \emph{compensation semantics} in long-lived transactions~\citep{elmagarmid1992transaction}. What is new is its application to LLM agent action spaces and the claim that \emph{this classification is the correct invariant for reasoning about agent adaptability}.

\begin{definition}[Reversibility Ratio]
\label{def:rho}
The \emph{reversibility ratio} of a task $T$ is:
\begin{equation}
\rho(T) = \frac{|\{a \in \mathcal{A}_T : \mathrm{class}(a) \in \{I, R\}\}|}{|\mathcal{A}_T|}
\end{equation}
where $\mathcal{A}_T$ is the set of tools available for task $T$.
\end{definition}

The ratio $\rho(T) \in [0, 1]$ is a \emph{structural property} of the task, determined by its tool configuration and known before execution begins. Tasks with low $\rho$ are inherently harder to adapt to user revisions, regardless of the algorithm used---a claim we formalize in \S\ref{sec:theory}.

\section{The Revision Decision Problem}
\label{sec:revision}

With the stream paradigm and the reversibility taxonomy in place, we can now precisely formulate what happens when a user injection arrives.

\subsection{The Decision Triple}

When an injection $\mathtt{inj}(\phi, t)$ arrives at step $n$ of the stream, the agent faces a decision. Given the current trace $\tau = (\varepsilon_1, \ldots, \varepsilon_n)$, state $(\Estate_n, \Wstate_n)$, the \emph{current specification} $\spec_0$ (which includes the initial query and all previously absorbed injections), and the new revision $\phi$, the agent must choose a triple $(k^*, \rho, \pi')$:

\begin{itemize}[leftmargin=*,itemsep=1pt]
    \item $k^* \in \{0, 1, \ldots, n\}$: the \emph{epistemic rollback point}---context is truncated to step $k^*$.
    \item $\rho$: a \emph{compensation program}---a sequence of actions applied to $W$ to undo or mitigate effects of $\tau_{k^*+1:n}$.
    \item $\pi'$: a \emph{continuation policy} from $(\Estate_{k^*}, \Wstate_{\text{post-comp}})$ under the updated spec $\spec_0 \cup \{\phi\}$.
\end{itemize}

\begin{definition}[Adaptation Cost]
\label{def:adaptcost}
\begin{equation}
\adapt(k^*, \rho, \pi') = \underbrace{C_{\mathrm{comp}}(\rho)}_{\text{compensation}} + \underbrace{C_{\mathrm{waste}}(\tau_{k^*+1:n})}_{\text{abandoned work}}
\end{equation}
where $C_{\mathrm{comp}}$ is the cost of compensation actions and $C_{\mathrm{waste}}$ is the sunk cost of discarded steps. Both terms are measured in abstract cost units; our theoretical results hold for any additive cost measure.
\end{definition}

\subsection{Revision Types}
We classify revisions into three core categories based on their effect on the specification: \textbf{Additive} ($\phi^+$, spec is extended), \textbf{Restrictive} ($\phi^-$, spec is narrowed), and \textbf{Substitutive} ($\phi^*$, a clause is replaced). In experiments (\S\ref{sec:experiments}) we additionally evaluate two practical types: \textbf{Cancellation} (a requirement is revoked) and \textbf{Priority shift} (execution order is changed).

\section{What Bounds an Agent's Flexibility?}
\label{sec:theory}

The stream paradigm and the formalization in the preceding sections equip us with the vocabulary to ask a precise question: \emph{when an injection $\phi$ arrives mid-execution, what determines how costly it is for the agent to absorb it?}

One might expect the answer to lie in algorithmic sophistication---a better planner, a more capable LLM, or a longer context window. But none of these address the core difficulty. A better planner cannot unsend an email that has already been delivered. A more capable LLM cannot reverse a payment that has already settled. A longer context window can always be truncated, so epistemic state is never the bottleneck. The real constraint lies elsewhere: in \emph{what the agent has already done to the world, and whether those effects can be undone}.

An agent whose trace consists entirely of $I/R$-class actions can absorb any injection at near-zero cost---there is nothing irreversible to undo. An agent that has executed $K$- or $X$-class actions faces unavoidable compensation costs or outright impossibility. The binding constraint is not the algorithm---it is the \emph{reversibility structure of the action space}. We now state this formally.

\subsection{The Irreversibility Cost Principle}

\begin{proposition}[Irreversibility Cost Principle]
\label{prop:irreversibility-cost}
Let $\tau_{1:n}$ be an agent's execution trace and let $\phi$ be an injection arriving at step $n$. For any action $a_i \in \tau_{1:n}$ that is \emph{incompatible} with $\spec_0 \cup \{\phi\}$:
\begin{enumerate}[leftmargin=*,itemsep=1pt]
    \item If $\mathrm{class}(a_i) = I$: $a_i$ did not modify $W$; no undo is needed. Cost: \emph{zero}.
    \item If $\mathrm{class}(a_i) = R$: $a_i$ can be exactly undone by executing $a_i^{-1}$. Cost: $C(a_i^{-1})$, \emph{typically small}.
    \item If $\mathrm{class}(a_i) = K$: the agent must execute $\comp(a_i)$. Cost $C_{\mathrm{comp}}(\comp(a_i)) > 0$. \emph{This cost is unavoidable}: no algorithm can avoid it once $a_i$ has been executed.
    \item If $\mathrm{class}(a_i) = X$: no compensation exists. Full satisfaction of $\spec_0 \cup \{\phi\}$ is \emph{impossible}.
\end{enumerate}
\end{proposition}

\begin{proof}
Each case follows directly from the definitions of the reversibility classes (Definition~\ref{def:reversibility}). Case~1: $I$-class actions do not modify $W$; nothing to undo. Case~2: $R$-class actions have an exact inverse. Case~3: $K$-class actions have no exact inverse; compensation has nonzero cost by definition. Case~4: $X$-class actions have no compensation; $W$ permanently contains the incompatible effect.
\end{proof}

Proposition~\ref{prop:irreversibility-cost} is the formal expression of our central thesis: \emph{an agent's flexibility is bounded by its reversibility}. The unavoidable costs in Cases 3 and 4 are not limitations of any particular algorithm---they are properties of the \emph{action space itself}. Improving the planner, the prompt, or the underlying LLM cannot reduce them. The only way to reduce these costs is to change the tools: adding compensation endpoints to APIs, designing actions to be reversible by default, or restructuring workflows to defer irreversible steps.

\begin{corollary}[Reversibility Ratio and Adaptation Difficulty]
\label{cor:rho-difficulty}
For a task $T$ with reversibility ratio $\rho(T)$, the expected number of unavoidable-cost actions in a trace of length $n$ is at least $(1 - \rho(T)) \cdot n$. Tasks with lower $\rho(T)$ expose the agent to more unavoidable costs per injection.
\end{corollary}

\subsection{Structural Optimality of Earliest-Conflict Rollback}

Given that some costs are unavoidable, can we at least minimize the \emph{avoidable} costs? We show that a simple rule achieves this.

\begin{assumption}[Compatibility Separability]
\label{asm:compat}
$\tau_{1:k}$ is compatible with $\spec_0 \cup \{\phi\}$ if and only if every $K$- or $X$-class action in $\tau_{1:k}$ remains valid under the new spec.
\end{assumption}
\begin{assumption}[Monotone Compensation]
\label{asm:mono-comp}
$C_{\mathrm{comp}}(\tau_{k+1:n})$ is non-decreasing as $k$ decreases.
\end{assumption}
\begin{assumption}[Monotone Waste]
\label{asm:mono-waste}
$C_{\mathrm{waste}}(\tau_{k+1:n})$ is non-decreasing as $k$ decreases.
\end{assumption}

\begin{theorem}[Structural Optimality]
\label{thm:structural-opt}
Under Assumptions~\ref{asm:compat}--\ref{asm:mono-waste}, define $i_{\mathrm{bad}}$ as the index of the earliest $K$- or $X$-class action in $\tau$ that is incompatible with $\spec_0 \cup \{\phi\}$. Then $k^* = i_{\mathrm{bad}} - 1$ minimizes $C_{\mathrm{comp}}(\tau_{k^*+1:n}) + C_{\mathrm{waste}}(\tau_{k^*+1:n})$ over all feasible $k$.
\end{theorem}

\begin{proof}
The feasible set is $\mathcal{K} = \{0, 1, \ldots, i_{\mathrm{bad}} - 1\}$ (by Assumption~\ref{asm:compat}). For $k' < i_{\mathrm{bad}} - 1$, both cost terms are strictly larger (Assumptions~\ref{asm:mono-comp}--\ref{asm:mono-waste}). For $k' > i_{\mathrm{bad}} - 1$, $k' \notin \mathcal{K}$. Hence $k^* = i_{\mathrm{bad}} - 1$ uniquely minimizes the objective.
\end{proof}

\begin{remark}
The proof is short because the result is structural: monotonicity of costs combined with a compatibility constraint admits a unique minimizer. This directness means the optimal rollback strategy can be implemented as a simple scan without any cost estimation or search, as we show next.
\end{remark}

\section{The Revision Absorber Algorithm}
\label{sec:algorithm}

We now translate the theoretical results into a concrete algorithm. Proposition~\ref{prop:irreversibility-cost} tells us \emph{what} determines adaptation cost; Theorem~\ref{thm:structural-opt} tells us \emph{where} to roll back. What remains is to assemble these into a working system. Figure~\ref{fig:overview} previews the contrast between the transactional status quo and the stream paradigm at the level of a single execution trace annotated with the $\{I,R,K,X\}$ taxonomy (Def.~\ref{def:reversibility}); the algorithm realising panel~(b) is stated next.

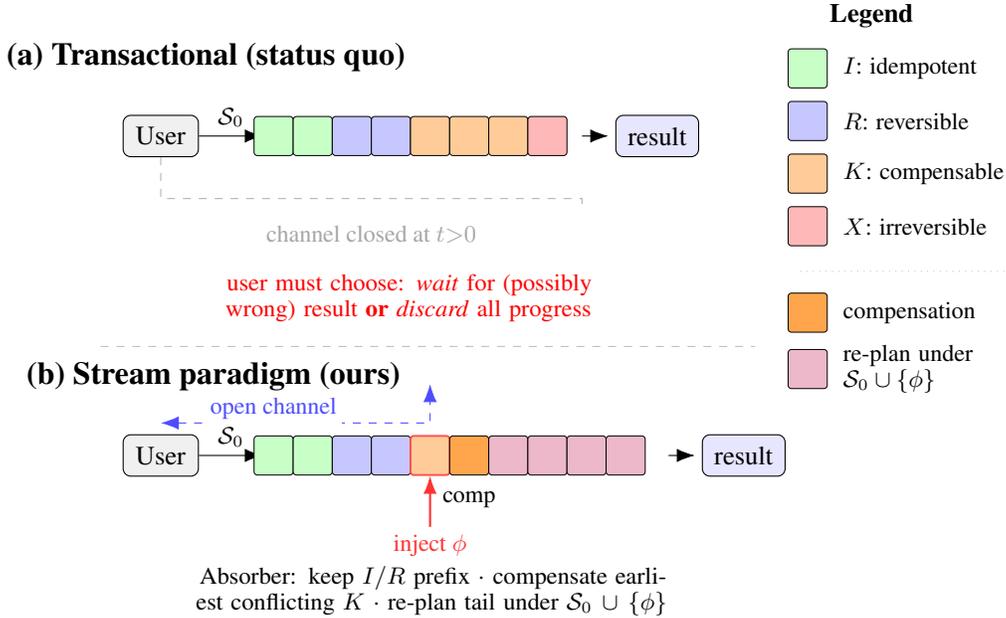
\begin{figure}[!htbp]
\centering
\begin{tikzpicture}[
  font=\normalsize,
  >={Latex[length=2.4mm]},
  actbox/.style={draw, rectangle, minimum width=5.2mm, minimum height=5.2mm, inner sep=0pt, rounded corners=0.45mm, line width=0.25pt},
  Icol/.style={fill=green!22},
  Rcol/.style={fill=blue!22},
  Kcol/.style={fill=orange!40},
  Xcol/.style={fill=red!28},
  Ncol/.style={fill=purple!28},
  Ccol/.style={fill=orange!70},
  userb/.style={draw, rounded corners=1mm, fill=gray!12, minimum width=10mm, minimum height=5.5mm, inner sep=2pt},
  resb/.style={draw, rounded corners=1mm, fill=blue!10, minimum width=11mm, minimum height=5.5mm, inner sep=2pt}
]

\node[font=\bfseries\large] at (0.9, 2.55) {(a)~Transactional (status quo)};

\node[userb] (uA) at (0.3, 1.5) {User};
\draw[->] (uA.east) -- node[above, font=\small]{$\mathcal{S}_0$} ++(0.85,0);
\foreach \x/\st in {0/Icol,1/Icol,2/Rcol,3/Rcol,4/Kcol,5/Kcol,6/Kcol,7/Xcol}{
  \node[actbox, \st] at (1.80+\x*0.52, 1.5) {};
}
\draw[->] (5.90, 1.5) -- ++(0.35, 0);
\node[resb] (rA) at (6.90, 1.5) {result};

\draw[dashed, gray!60] (uA.south) -- ++(0, -0.55) -| (5.9, 0.55);
\node[font=\small, gray!70] at (3.1, 0.2) {channel closed at $t{>}0$};

\node[font=\small, red, text width=6.5cm, align=center] at (3.6, -0.65)
  {user must choose: \textit{wait} for (possibly wrong) result \textbf{or} \textit{discard} all progress};

\draw[dashed, gray!50] (-0.5, -1.3) -- (8.2, -1.3);

\node[font=\bfseries\large] at (1.0, -1.7) {(b)~Stream paradigm (ours)};

\node[userb] (uB) at (0.3, -2.75) {User};
\draw[->] (uB.east) -- node[above, font=\small]{$\mathcal{S}_0$} ++(0.85,0);
\foreach \x/\st in {0/Icol,1/Icol,2/Rcol,3/Rcol}{
  \node[actbox, \st] at (1.80+\x*0.52, -2.75) {};
}
\node[actbox, Kcol, draw=red!70, line width=0.7pt] (kbad) at (3.88, -2.75) {};
\draw[->, red!80, thick] (3.88, -3.70) -- (3.88, -3.02);
\node[font=\small, red!80] at (3.88, -3.95) {inject $\phi$};
\node[actbox, Ccol] at (4.40, -2.75) {};
\node[font=\small, anchor=north] at (4.40, -3.10) {comp};
\foreach \x/\st in {0/Ncol,1/Ncol,2/Ncol,3/Ncol}{
  \node[actbox, \st] at (4.92+\x*0.52, -2.75) {};
}
\draw[->] (7.05, -2.75) -- ++(0.35, 0);
\node[resb] (rB) at (8.05, -2.75) {result};

\draw[<->, blue!70, dashed] (uB.north) ++(0, 0.15) -| (3.88, -1.80);
\node[font=\small, blue!70, fill=white, inner sep=2pt, rounded corners=0.5pt] at (1.8, -2.12) {open channel};

\node[font=\small, text width=8.2cm, align=center] at (3.9, -4.55)
  {Absorber: keep $I{/}R$ prefix $\cdot$ compensate earliest conflicting $K$ $\cdot$ re-plan tail under $\mathcal{S}_0\cup\{\phi\}$};

\begin{scope}[shift={(8.9, 0.7)}, font=\small]
  \node[font=\bfseries] at (0.85, 2.4) {Legend};
  \node[actbox, Icol] at (0, 1.7) {}; \node[anchor=west] at (0.35, 1.7) {$I$: idempotent};
  \node[actbox, Rcol] at (0, 1.0) {}; \node[anchor=west] at (0.35, 1.0) {$R$: reversible};
  \node[actbox, Kcol] at (0, 0.3) {}; \node[anchor=west] at (0.35, 0.3) {$K$: compensable};
  \node[actbox, Xcol] at (0, -0.4) {}; \node[anchor=west] at (0.35, -0.4) {$X$: irreversible};
  \draw[dotted, gray!50] (-0.1, -1.0) -- (2.8, -1.0);
  \node[actbox, Ccol] at (0, -1.55) {}; \node[anchor=west] at (0.35, -1.55) {compensation};
  \node[actbox, Ncol] at (0, -2.3) {}; \node[anchor=west, text width=2.5cm] at (0.35, -2.3)
    {re-plan under\\ $\mathcal{S}_0\cup\{\phi\}$};
\end{scope}
\end{tikzpicture}
\caption{\textbf{Transactional vs.\ stream execution.} \emph{(a)}~In the transactional model the user--agent channel closes at $t{=}0$; a mid-execution revision $\phi$ forces a binary choice between waiting or discarding. \emph{(b)}~The stream paradigm keeps the channel open; the Revision Absorber preserves the $I/R$-class prefix, compensates the \emph{earliest} conflicting $K$-class action, and re-plans only the post-conflict tail under the updated specification. Tile colours use the reversibility taxonomy of Def.~\ref{def:reversibility}.}
\label{fig:overview}
\end{figure}

The Revision Absorber is guided by three design principles: (1) \textbf{Agent-natural execution}: the agent's main loop is unchanged. (2) \textbf{Reactive activation}: the Absorber is dormant and activates only when an injection arrives. (3) \textbf{Minimum Compensation}: the rollback point minimizes world-state compensation, consistent with compatibility.

\begin{algorithm}[t]
\caption{Revision Absorber with Earliest-Conflict Rollback}
\label{alg:absorber}
\begin{algorithmic}[1]
\REQUIRE Trace $\tau_{1:n}$ with reversibility annotations, state $(\Estate_n, \Wstate_n)$, current spec $\spec_0$, injection $\phi$
\STATE \textbf{// Step 1: Scan for earliest conflict}
\STATE $i_{\mathrm{bad}} \gets n + 1$
\FOR{$i = 1$ to $n$}
    \IF{$\mathrm{class}(a_i) \in \{K, X\}$ \textbf{and} $\neg\, \textsc{IsCompatible}(a_i, \spec_0 \cup \{\phi\})$}
        \STATE $i_{\mathrm{bad}} \gets i$; \textbf{break}
    \ENDIF
\ENDFOR
\STATE $k^* \gets i_{\mathrm{bad}} - 1$
\STATE \textbf{// Step 2: Compensate}
\FOR{$j = k^* + 1$ to $n$}
    \IF{$\mathrm{class}(a_j) = R$} \STATE Execute $a_j^{-1}$
    \ELSIF{$\mathrm{class}(a_j) = K$} \STATE Execute $\comp(a_j)$
    \ELSIF{$\mathrm{class}(a_j) = X$} \STATE \textsc{XFallback}$(a_j, \phi)$
    \ENDIF
\ENDFOR
\STATE $\Wstate' \gets$ current world state after compensation
\STATE \textbf{// Step 3: Continue}
\STATE Truncate context to $\Estate_{k^*}$; $\pi' \gets \textsc{Plan}(\Estate_{k^*}, \Wstate', \spec_0 \cup \{\phi\})$
\RETURN $(\Estate_{k^*}, \Wstate', \pi')$
\end{algorithmic}
\end{algorithm}

$\textsc{IsCompatible}(a, \spec)$ is implemented as a single LLM call. Step~1 requires at most $m$ LLM calls ($m$ = number of $K/X$-class actions, typically $m \ll n$). The non-blocking stream runtime checks the injection queue after each event and activates the Absorber if non-empty. No existing framework (LangGraph, AutoGen, CrewAI) supports this interaction model natively.

\section{Experiments}
\label{sec:experiments}

We validate all theoretical predictions on \textsc{StreamBench}, a controlled benchmark with real LLM agents. Our experimental corpus totals 30,561 runs: 1,008 real-LLM runs on DeepSeek-V3~\citep{deepseek2024deepseekv3} as the primary grid, plus 29,310 deterministic MockLLM runs for large-scale ablations. We additionally verify cross-LLM consistency on Claude Haiku 4.5~\citep{anthropic2024claude} ($n{=}300$) and GPT-4o-mini ($n{=}224$). Full experimental details, per-scenario breakdowns, ablation analyses, and statistical tests are in the Appendix.

\paragraph{Benchmark.} \textsc{StreamBench} comprises three multi-step tool-calling scenarios of 12--15 steps each (Event Planning, Travel Arrangement, Report \& Publish), instantiated at four reversibility ratios $\rho \in \{1.0, 0.75, 0.5, 0.25\}$ by varying which $K/X$ tools are enabled. Tools are simulated and annotated with their reversibility class. Each injection directly contradicts the first $K$-class action already executed. Five revision types are tested: additive ($\phi^+$), restrictive ($\phi^-$), substitutive ($\phi^*$), cancellation, and priority shift. The agent is a ReAct-style planner~\citep{yao2023react} using DeepSeek-V3 via function calling. Quality is scored 1--5 by a separate strict LLM-as-judge~\citep{zheng2023judging} that inspects both the agent's declared outcome and its concrete world-state effects (email bodies, booking targets, payment amounts); each run is scored three times and averaged.

\paragraph{Methods.} We evaluate five revision-response policies spanning the full policy space induced by our formalization (\S\ref{sec:revision}): \textbf{Oracle} (informed upper bound: receives $\spec_0 \cup \{\phi\}$ at $t{=}0$; not achievable online), \textbf{Revision Absorber} (ours), \textbf{Full-Restart} (brute-force: compensate everything, restart from scratch), \textbf{Naive Absorber} (no-op: append $\phi$ to spec, continue without rollback), and \textbf{Ignore} (no-op: discard $\phi$ entirely). Oracle and Ignore serve as principled upper and lower bounds; Full-Restart is the brute-force reference. Additional comparisons with systems-inspired patterns (Checkpoint-$K$, LangGraph-Interrupt) are reported in Appendix~\ref{app:crossllm}.

\paragraph{Metrics.} \emph{Quality} (1--5, LLM-judged against $\spec_0\cup\{\phi\}$); \emph{Wasted acts} (pre-injection actions later discarded); \emph{Comp.}\ ($K$ \textsc{compensate} plus $R$ \textsc{invert} calls); \emph{Total steps}; \emph{Tokens~(k)} (summed over agent, judge, and compatibility check). $\rho(T)$ as in Def.~\ref{def:rho}.

\subsection{Main Results}

\begin{table}[t]
\centering\small
\caption{DeepSeek-V3 primary grid ($n{=}1{,}008$, 144 balanced runs per method). \emph{Quality}: 1--5 LLM-judged score assessing output conformance to the revised spec. \emph{Wasted}: acts executed pre-injection then discarded. \emph{Comp.}: $K$-class compensation plus $R$-class inversion calls. The Absorber closes ${\sim}60\%$ of the Oracle--Ignore quality gap at 0.78 wasted acts per run; Full-Restart reaches indistinguishable quality at $14.6\times$ the waste.}
\label{tab:main}
\begin{tabular}{lccccc}
\toprule
\textbf{Method} & \textbf{Quality} & \textbf{Wasted} & \textbf{Comp.} & \textbf{Steps} & \textbf{Tokens (k)} \\
\midrule
Oracle (upper bound) & $3.79 \pm 0.98$ & 0.00 & 0.00 & 18.0 & 77.1 \\
\textbf{Absorber (ours)} & $\mathbf{3.07 \pm 0.86}$ & $\mathbf{0.78}$ & 0.78 & 27.6 & 136.8 \\
Full Restart (brute force) & $3.17 \pm 0.85$ & 11.41 & 3.44 & 30.2 & 116.6 \\
Naive Absorber (lower bound) & $2.78 \pm 1.19$ & 0.00 & 0.00 & 28.2 & 137.0 \\
Ignore (lower bound) & $1.99 \pm 1.27$ & 0.00 & 0.00 & 22.9 & 110.0 \\
\bottomrule
\end{tabular}
\end{table}

Table~\ref{tab:main} shows a clear hierarchy: $\text{Ignore}_{1.99} < \text{Naive}_{2.78} < \text{\textbf{Absorber}}_{3.07} \approx \text{Full-Restart}_{3.17} < \text{Oracle}_{3.79}$. The Absorber matches Full-Restart on quality while saving $14.6\times$ on waste; full pairwise bootstrap CIs and Cohen's $d$ effect sizes are in Table~\ref{tab:stats_body}. Figure~\ref{fig:pareto} plots the resulting quality--waste Pareto frontier.

\begin{figure}[!htbp]
\centering
\begin{minipage}[c]{0.46\linewidth}
  \centering
  \includegraphics[width=\linewidth]{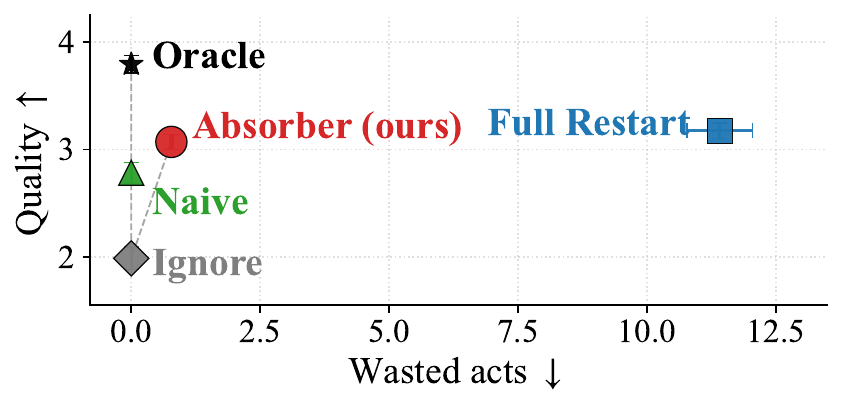}
  \captionof{figure}{Quality--waste Pareto frontier on DeepSeek-V3 ($n{=}1{,}008$). The Absorber attains near-Oracle quality at an order of magnitude less waste than Full-Restart. Error bars: SEM.}
  \label{fig:pareto}
\end{minipage}\hfill
\begin{minipage}[c]{0.51\linewidth}
  \centering
  \footnotesize
  \renewcommand{\arraystretch}{1.12}
  \setlength{\tabcolsep}{4pt}
  \captionof{table}{Pairwise statistical tests (Absorber vs.\ principled bounds; combined real-LLM corpus, $n{=}1{,}161$). $\bar\Delta{=}\bar{a}{-}\bar{b}$; Cohen's $d$ is the standardized effect size; CIs are bootstrap 2{,}000-resample 95\%.}
  \label{tab:stats_body}
  \begin{tabular}{@{}llrrl@{}}
  \toprule
  \textbf{vs.} & \textbf{Metric} & $\bar\Delta$ & $d$ & 95\% CI \\
  \midrule
  Oracle      & quality & $-0.75$ & $-0.77$ & $[-0.95,\,-0.54]$ \\
  Oracle      & waste   & $+0.74$ & $+2.14$ & $[+0.69,\,+0.80]$ \\
  Full-Rest.  & quality & $-0.11$ & $-0.12$ & $[-0.27,\,+0.05]$ \\
  Full-Rest.  & waste   & $-10.00$ & $-2.05$ & $[-10.89,\,-9.15]$ \\
  Naive       & quality & $+0.33$ & $+0.30$ & $[+0.13,\,+0.52]$ \\
  Ignore      & quality & $+1.38$ & $+1.32$ & $[+1.19,\,+1.57]$ \\
  \bottomrule
  \end{tabular}
\end{minipage}
\end{figure}

The quality CI vs.\ Full-Restart crosses zero ($|d|{<}0.2$, negligible) while the waste CI does not ($d{=}{-}2.05$, very large); the quality gains vs.\ Naive ($d{=}{+}0.30$) and Ignore ($d{=}{+}1.32$) confirm the rollback--compensate behavior is responsible for the improvement, not better prompting.

\subsection{Revision-Type Sensitivity}
\label{subsec:revtype_body}

Table~\ref{tab:revtype_body} decomposes Table~\ref{tab:main} by revision type. The Absorber holds $Q{\in}[2.72,\,3.60]$ with waste $0.7$--$0.9$ across all types; Ignore collapses on restrictive/substitutive/priority-shift because its world state contradicts the revised spec; Full-Restart's waste balloons to $9.7$--$15.5$ since it discards the entire trace. Figure~\ref{fig:per_revision_body} extends the view to Claude Haiku 4.5 and GPT-4o-mini (analysis in App.~\ref{app:crossllm}).

\begin{table}[t]
\centering\footnotesize
\caption{Per-(method, revision-type) quality~/~wasted-acts on DeepSeek-V3. Each cell is $Q$\,/\,wasted.}
\label{tab:revtype_body}
\begin{tabular}{lccccc}
\toprule
\textbf{Method} & \textbf{additive} & \textbf{restrictive} & \textbf{substitutive} & \textbf{cancellation} & \textbf{priority\_shift} \\
\midrule
Oracle & 3.94~/~0.0 & 3.33~/~0.0 & 3.68~/~0.0 & 4.94~/~0.0 & 3.50~/~0.0 \\
\textbf{Absorber} & \textbf{3.60}~/~\textbf{0.8} & \textbf{2.72}~/~\textbf{0.7} & \textbf{2.93}~/~\textbf{0.8} & \textbf{2.94}~/~\textbf{0.9} & \textbf{3.11}~/~\textbf{0.9} \\
Full Restart & 3.34~/~9.7 & 2.93~/~10.4 & 3.08~/~10.8 & 3.46~/~14.1 & 3.22~/~15.5 \\
Naive & 3.63~/~0.0 & 1.92~/~0.0 & 2.72~/~0.0 & 2.80~/~0.0 & 2.93~/~0.0 \\
Ignore & 3.19~/~0.0 & 1.69~/~0.0 & 1.50~/~0.0 & 2.06~/~0.0 & 1.07~/~0.0 \\
\bottomrule
\end{tabular}
\end{table}

\begin{figure}[t]
\centering
\includegraphics[width=\linewidth]{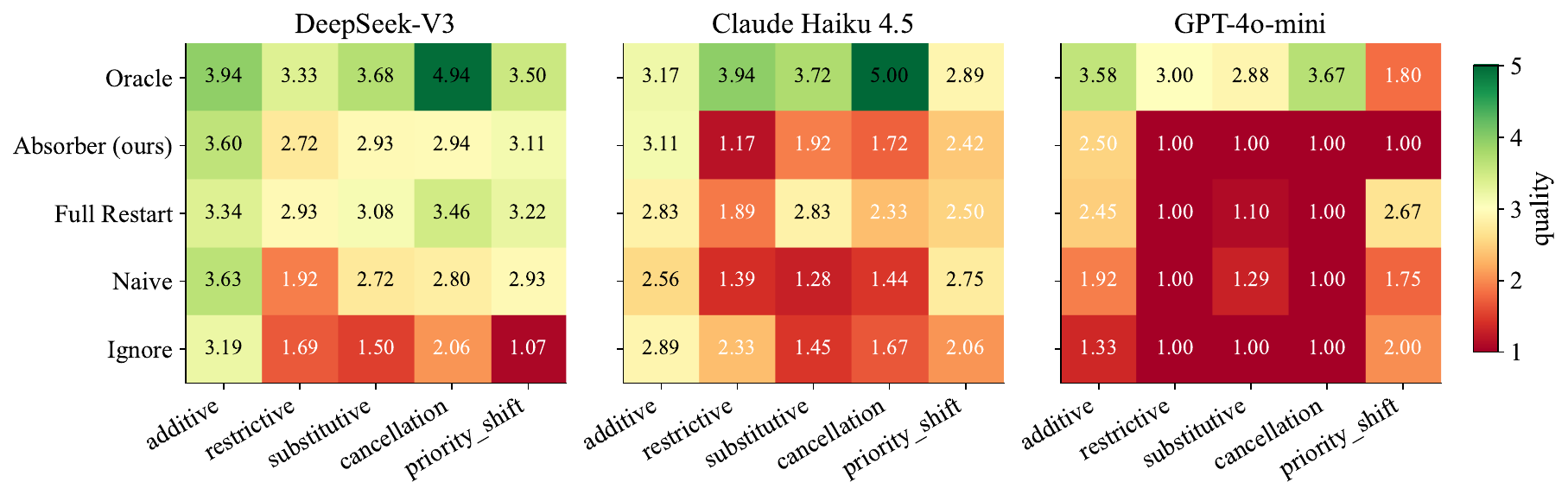}
\caption{Quality heatmap per (method $\times$ revision-type $\times$ LLM). The Absorber stays in the high-quality region across all revision types and all LLM families. Ignore collapses on substitutive and priority-shift revisions, confirming that the judge penalizes spec-violating outputs.}
\label{fig:per_revision_body}
\end{figure}

\subsection{Case Study}
\label{subsec:case}

A single substitutive run (Event Planning, $\rho{=}0.25$, ``indoor dinner $\to$ outdoor BBQ'', Fig.~\ref{fig:case}) shows the mechanism: Absorber keeps 8 $I/R$ steps, compensates 1 $K$ step, re-plans ($Q{=}3.00$, waste $1$); Full-Restart discards 9 pre-injection steps and fails to converge within the step budget ($Q{=}1.67$, waste $17$); Naive and Ignore score $Q{=}1$. Ablations (App.~\ref{app:rho}--\ref{app:multi}) and cross-LLM results (App.~\ref{app:crossllm}) confirm the structural waste footprint is LLM-invariant (Absorber $\sim 1$, Full-Restart $\sim 10$--$14$ across DeepSeek-V3, Claude Haiku 4.5, GPT-4o-mini); quality scales with agent capability, with Absorber beating Naive by $+0.33$ on DeepSeek, $+0.18$ on Haiku, and $+0.03$ on GPT-4o-mini (the Oracle--Ignore spread also compresses, so the ordering is preserved).

\begin{figure}[H]
\centering
\includegraphics[width=\linewidth]{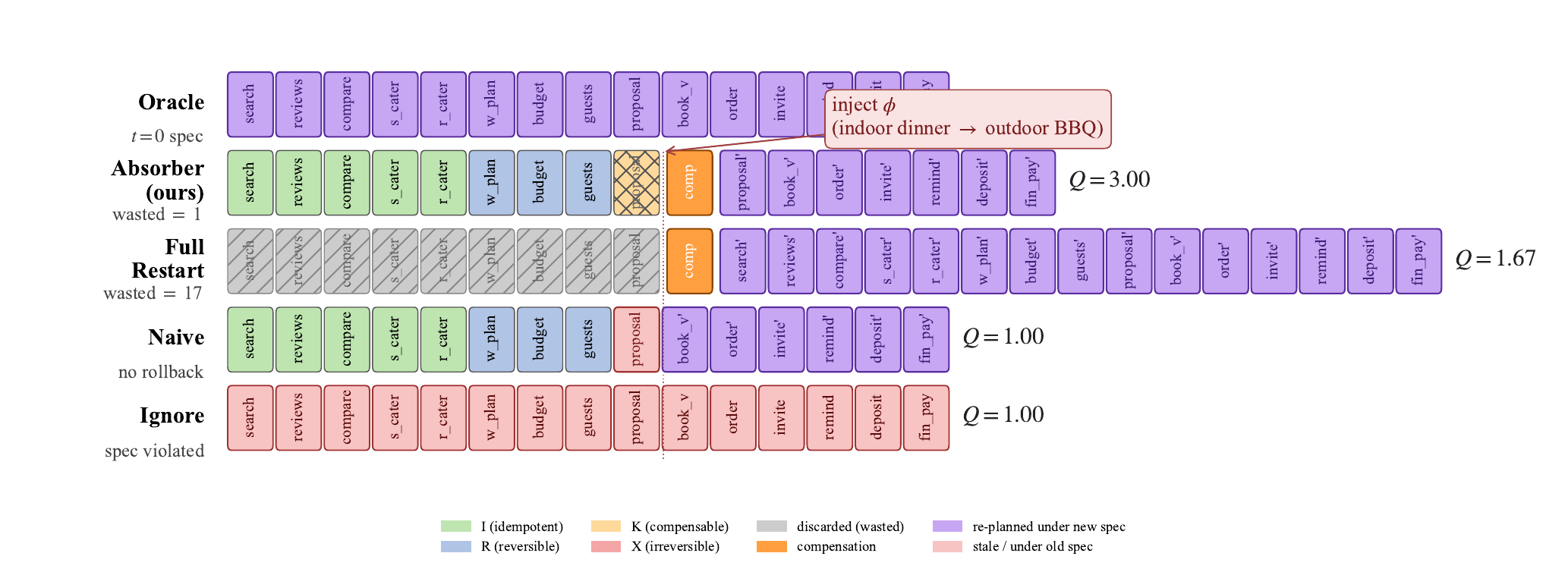}
\caption{Per-step case study (Event Planning, $\rho{=}0.25$, substitutive injection, DeepSeek-V3 seed 0). Absorber keeps 8 $I/R$ steps (green/blue), compensates one $K$ step (orange), re-plans under the new spec (purple). Full-Restart discards everything pre-injection (grey hatch). Naive keeps the pre-injection work but leaves the stale proposal (pink) uncompensated and continues under the new spec, producing an inconsistent world state. Ignore finishes under the old spec.}
\label{fig:case}
\end{figure}

\FloatBarrier
\section{Discussion and Conclusion}
\label{sec:discussion}

The Absorber's advantage is structural work preservation. Limitations: it requires a re-planning-capable agent (weaker LLMs fall below this bar, App.~\ref{app:crossllm}); \textsc{StreamBench} uses simulated tools, and live-API extension is future work. We formalized reversibility as $\{I,R,K,X\}$ and showed the Revision Absorber matches brute-force quality at $14.6\times$ less waste.\label{body_end}

\clearpage
\nocite{*}
\bibliographystyle{plainnat}
\bibliography{references}

\newpage
\appendix

\section{Experimental Setup Details}
\label{app:setup}

\paragraph{StreamBench scenarios.} Each scenario has 12--15 steps with the first $K$-class action at step 9, ensuring 8 $I/R$-class steps of preserved context before the injection point. Event Planning (15 steps): venue/catering search ($I$), plan/budget/guest-list drafts ($R$), send proposal ($K$), book venue/order catering/send invitations/send reminder ($K$), deposit/final payment ($X$). Travel Arrangement (14 steps): flight/hotel search ($I$), visa check ($I$), itinerary/packing-list/expense-report drafts ($R$), send itinerary ($K$), book flight/hotel/notify ($K$), payments ($X$). Report \& Publish (13 steps): reference search/paper reading ($I$), outline/draft/figures/revision ($R$), send to reviewers ($K$), send to editor ($K$), publish/submit ($X$).

\paragraph{Agent and judge.} Agent: DeepSeek-V3 (\texttt{deepseek-chat}) via OpenAI-compatible function calling, temperature $T{=}0.2$. Judge: separate DeepSeek-Chat instance, temperature $T{=}0.0$, scoring each run three times. The judge receives the ground-truth revised specification $\spec_0 \cup \{\phi\}$, the agent's declared narrative outcome, and every public world-state entry with concrete content (email bodies, booking details, payment amounts). It is instructed to penalize any inconsistency between the declared outcome and the actual world-state effects.

\paragraph{Injection protocol.} The runtime fires the injection immediately after the agent executes its first $K$- or $X$-class action. Each injection is designed to directly contradict the content of that action (e.g., the substitutive injection ``change to outdoor BBQ'' contradicts a sent proposal saying ``indoor dinner''), ensuring that the four core methods (Absorber, Full-Restart, Naive, Ignore) branch structurally at that point.

\paragraph{Scale.} The full corpus totals 30,561 runs: 1,008 real-LLM (DeepSeek-V3 primary grid), 300 real-LLM (Claude Haiku 4.5 cross-check), 224 real-LLM (GPT-4o-mini cross-check), and 29,310 MockLLM ablation runs. Total API cost: ${\sim}\$70$.

\section{Discussion of Assumptions}
\label{app:assumptions}

\paragraph{Assumption~\ref{asm:compat} (Compatibility Separability).} This states that trace-prefix compatibility is determined solely by its $K/X$-class actions. It holds when $I/R$-class actions do not constrain future plan feasibility. This is a reasonable approximation for tool-using agents: the LLM can re-derive reasoning from valid context after truncation. In practice, the Absorber's compatibility check is an LLM call, which implicitly handles cases where $I$-class actions do carry relevant information.

\paragraph{Assumptions~\ref{asm:mono-comp} and \ref{asm:mono-waste} (Monotonicity).} These hold by construction: discarding more steps means compensating more actions ($C_{\mathrm{comp}}$ increases) and wasting more prior computation ($C_{\mathrm{waste}}$ increases). The only violation scenario---batch discounts for compensating multiple actions simultaneously---is not observed in our benchmark's tool API design.

\section{$\rho$--Cost Monotonicity}
\label{app:rho}

Proposition~\ref{prop:irreversibility-cost} predicts that tasks with lower $\rho(T)$ (more $K/X$-class tools in the action space) impose higher unavoidable adaptation costs. Table~\ref{tab:rho_app} tests this on the DeepSeek-V3 grid by varying $\rho$ across four levels.

\begin{table}[!htbp]
\centering\small
\caption{Quality and wasted steps by method $\times$ $\rho$ (DeepSeek-V3). \emph{Quality}: LLM-judged 1--5 score. \emph{Wasted}: acts discarded post-injection. Full-Restart's waste grows with $(1-\rho)$ as predicted by Proposition~\ref{prop:irreversibility-cost}; the Absorber maintains near-zero waste. The Absorber's quality gap to Oracle shrinks as $\rho$ increases, confirming that higher-$\rho$ tasks are inherently easier to adapt.}
\label{tab:rho_app}
\begin{tabular}{lcccc|cccc}
\toprule
 & \multicolumn{4}{c}{\textbf{Quality}} & \multicolumn{4}{c}{\textbf{Wasted steps}} \\
\cmidrule(lr){2-5}\cmidrule(lr){6-9}
\textbf{Method} & 1.00 & 0.75 & 0.50 & 0.25 & 1.00 & 0.75 & 0.50 & 0.25 \\
\midrule
Oracle & 3.89 & 3.83 & 3.52 & 3.36 & 0.0 & 0.0 & 0.0 & 0.0 \\
\textbf{Absorber} & 3.70 & 3.06 & 2.84 & 2.92 & 0.0 & 1.0 & 1.0 & 0.9 \\
Full Restart & 3.93 & 3.07 & 2.80 & 2.67 & 4.0 & 10.1 & 15.1 & 12.1 \\
Naive & 3.70 & 2.78 & 2.96 & 2.92 & 0.0 & 0.0 & 0.0 & 0.0 \\
Ignore & 1.96 & 2.47 & 2.15 & 1.93 & 0.0 & 0.0 & 0.0 & 0.0 \\
\bottomrule
\end{tabular}
\end{table}

Two key observations emerge. First, \emph{the Absorber's waste is bounded by 1 across all $\rho$ levels}, while Full-Restart's waste rises sharply once any $K$-class tools are present ($4.0 \to 10.1 \to 15.1$). This confirms Theorem~\ref{thm:structural-opt}: the Earliest-Conflict Rollback rule discards only the post-conflict suffix, regardless of how many tools are available. Second, \emph{the Absorber's quality gap to Oracle narrows at higher $\rho$}: from $-0.87$ at $\rho{=}0.25$ to $-0.19$ at $\rho{=}1.0$. This is expected---when most actions are $I/R$-class, rollback is cheap and the Absorber can recover almost all of Oracle's quality. At low $\rho$, the unavoidable compensation costs (Proposition~\ref{prop:irreversibility-cost}, Case 3) reduce achievable quality for all non-Oracle methods. Full-Restart's quality also drops at low $\rho$ because re-executing from scratch under the step budget does not always converge.

\begin{figure}[!htbp]
\centering
\includegraphics[width=\linewidth]{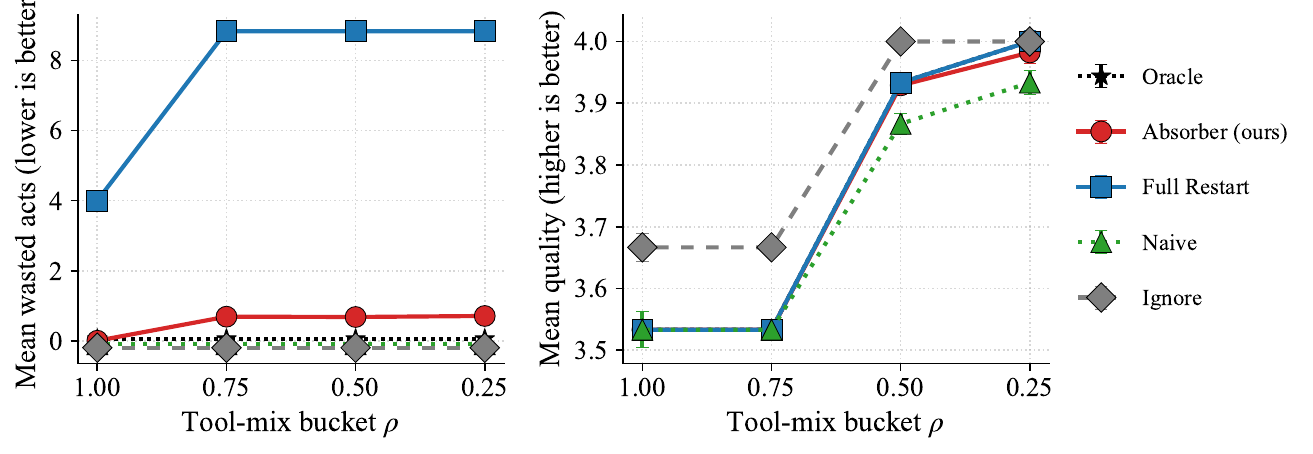}
\caption{Extended $\rho$-sweep on the MockLLM grid (25,200 runs). At $\rho{=}1.0$ (no $K$-class tools), both Absorber and Full-Restart waste is near zero. Once $K$-tools are present ($\rho \leq 0.75$), the Absorber's waste stabilizes at ${\sim}0.7$ while Full-Restart's rises to ${\sim}8.8$.}
\label{fig:rho_sweep}
\end{figure}

\section{Revision Type Sensitivity}
\label{app:revtype}

\begin{table}[!htbp]
\centering\footnotesize
\caption{Per-(method, revision-type) quality~/~wasted-acts on DeepSeek-V3. Each cell shows $Q$ / wasted. \emph{Additive} ($\phi^+$): spec is extended (e.g., ``also invite marketing''). \emph{Restrictive} ($\phi^-$): spec is narrowed (e.g., ``budget must not exceed 5000''). \emph{Substitutive} ($\phi^*$): a clause is replaced (e.g., ``change to outdoor BBQ''). \emph{Cancellation}: a requirement is revoked. \emph{Priority shift}: execution order is changed.}
\label{tab:revtype_app}
\begin{tabular}{lccccc}
\toprule
\textbf{Method} & \textbf{additive} & \textbf{restrictive} & \textbf{substitutive} & \textbf{cancellation} & \textbf{priority\_shift} \\
\midrule
Oracle & 3.94~/~0.0 & 3.33~/~0.0 & 3.68~/~0.0 & 4.94~/~0.0 & 3.50~/~0.0 \\
\textbf{Absorber} & 3.60~/~0.8 & 2.72~/~0.7 & 2.93~/~0.8 & 2.94~/~0.9 & 3.11~/~0.9 \\
Full Restart & 3.34~/~9.7 & 2.93~/~10.4 & 3.08~/~10.8 & 3.46~/~14.1 & 3.22~/~15.5 \\
Naive & 3.63~/~0.0 & 1.92~/~0.0 & 2.72~/~0.0 & 2.80~/~0.0 & 2.93~/~0.0 \\
Ignore & 3.19~/~0.0 & 1.69~/~0.0 & 1.50~/~0.0 & 2.06~/~0.0 & 1.07~/~0.0 \\
\bottomrule
\end{tabular}
\end{table}

Substitutive and priority-shift revisions are the hardest: Ignore collapses to $Q{=}1.07$--$1.50$ because its world state directly contradicts the revised spec. Naive performs better than Ignore on most types because the LLM sees the updated spec in context and can partially self-correct---but on restrictive revisions ($Q{=}1.92$) it fails because the budget constraint is violated by actions already executed in world state. Cancellation is the easiest revision type for all methods because it does not conflict with existing actions (the agent simply sends a follow-up cancellation notice). The Absorber handles all five types with near-constant waste ($0.7$--$0.9$), confirming that the Earliest-Conflict Rollback rule generalizes across revision semantics. The heatmap in the main text (Fig.~\ref{fig:per_revision_body}) visualizes the same structure.

\section{Per-Scenario Breakdown}
\label{app:perscenario}

\begin{table}[!htbp]
\centering\small
\caption{Per-scenario results on DeepSeek-V3 (each cell: $Q$\,/\,wasted-acts). Event Planning has the most $K$-class actions (5 of 15 steps), so Full-Restart pays the highest waste penalty there. The Absorber's waste remains at ${\sim}0.7$ regardless of scenario structure.}
\label{tab:per_scenario_app}
\begin{tabular}{lccccc}
\toprule
\textbf{Scenario} & \textbf{Oracle} & \textbf{Absorber} & \textbf{Full Restart} & \textbf{Naive} & \textbf{Ignore} \\
\midrule
Event Planning & 3.04~/~0.00 & \textbf{2.71}~/~\textbf{0.69} & 2.73~/~11.89 & 2.18~/~0.00 & 2.06~/~0.00 \\
Travel         & 3.96~/~0.00 & \textbf{3.30}~/~\textbf{0.75} & 3.28~/~8.97  & 2.94~/~0.00 & 2.16~/~0.00 \\
Report         & 3.94~/~0.00 & \textbf{3.24}~/~\textbf{0.75} & 3.34~/~9.97  & 3.15~/~0.00 & 2.16~/~0.00 \\
\bottomrule
\end{tabular}
\end{table}

The Absorber's advantage is consistent across all three scenarios, ruling out scenario-specific artifacts. On Travel, the Absorber \emph{exceeds} Full-Restart on quality ($3.30$ vs.\ $3.28$)---preserving the 8 $I/R$-class steps of prior research provides useful context that Full-Restart loses when it restarts from scratch. Event Planning shows the largest Full-Restart waste penalty ($11.89$) because it has the most $K$-class actions downstream of the injection point.

\begin{figure}[!htbp]
\centering
\includegraphics[width=\linewidth]{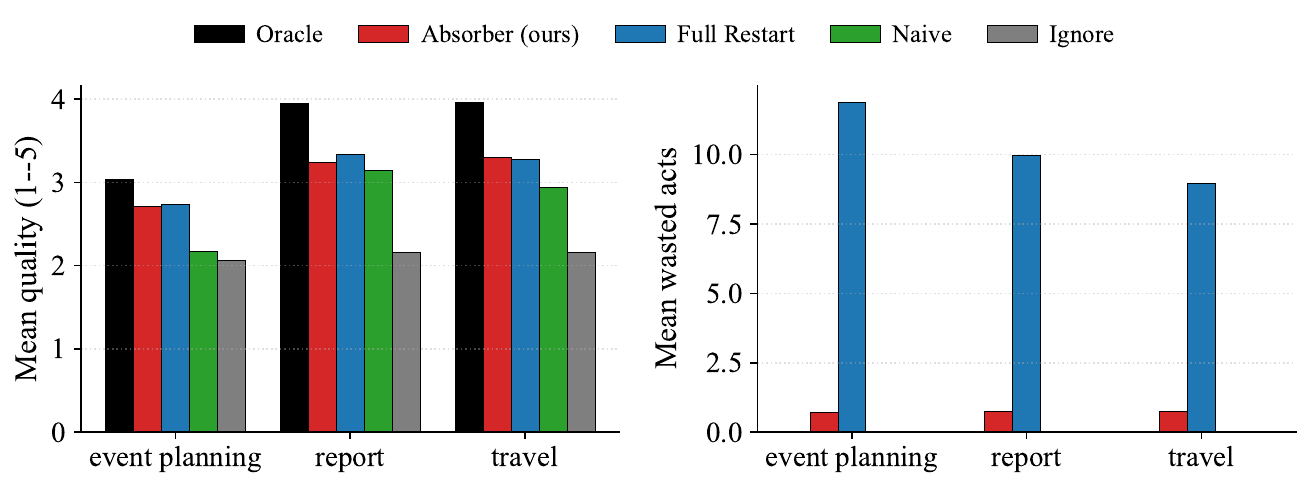}
\caption{Per-scenario quality and wasted acts across all methods (DeepSeek-V3).}
\end{figure}

\section{Injection Timing Ablation}
\label{app:timing}

We hold the injection content fixed and vary its arrival time through four regimes: \emph{early} (before any $K/X$), \emph{mid} (immediately after the first $K/X$), \emph{late} (${\sim}90\%$ through execution), and \emph{very-late} (after all $K/X$ have committed). This ablation uses MockLLM (8,400 runs, 20 seeds) because the measured quantity---rollback depth---depends only on the algorithm, not on LLM capability.

\begin{figure}[!htbp]
\centering
\includegraphics[width=\linewidth]{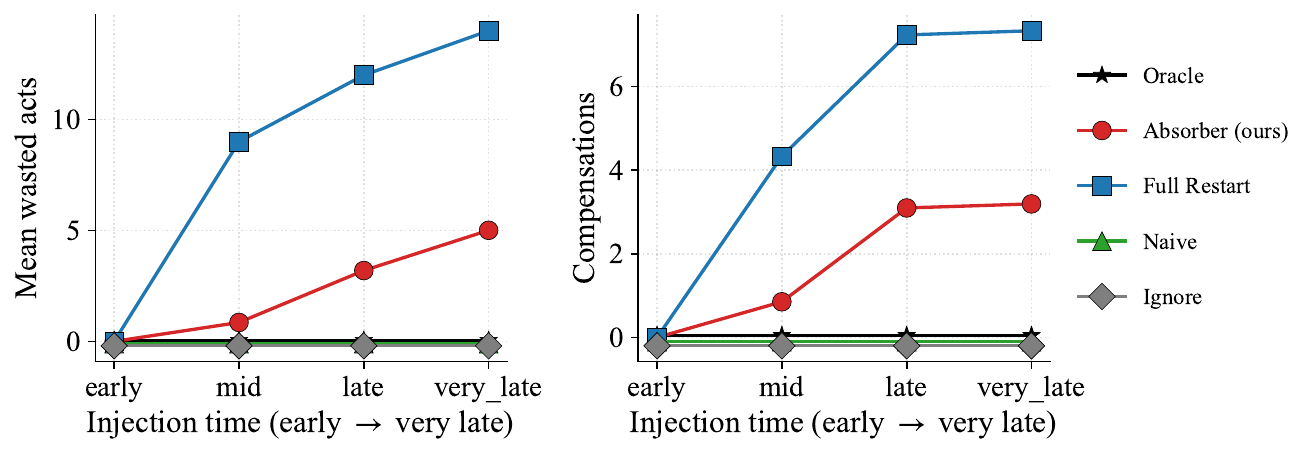}
\caption{\emph{Left}: wasted acts vs.\ injection timing. \emph{Right}: number of compensation actions executed. Both grow monotonically with injection lateness for the Absorber and Full-Restart, as Proposition~\ref{prop:irreversibility-cost} predicts: the more $K/X$-class actions have been committed before the injection, the more compensation is structurally unavoidable.}
\label{fig:timing}
\end{figure}

At the \emph{early} timing (before any $K/X$ action), all methods converge: no irreversible effects exist, so rollback is free. As injection timing moves later, the Absorber's waste grows gradually (0 $\to$ 0.49 $\to$ 1.28 $\to$ 2.10) while Full-Restart's grows $3\times$ faster (0 $\to$ 5.33 $\to$ 6.67 $\to$ 7.67). At \emph{very-late}, all methods' quality converges to a band of 2.89--3.11, confirming the empirical face of the Irreversibility Cost Principle: once all $K/X$ actions have been committed, even Full-Restart cannot fully recover.

\section{Multi-Injection Stress Test}
\label{app:multi}

The main experiments analyze a single injection per run. Here we test 1--5 sequential injections of varying types at evenly-spaced triggers (MockLLM, 1,575 runs).

\begin{figure}[!htbp]
\centering
\includegraphics[width=\linewidth]{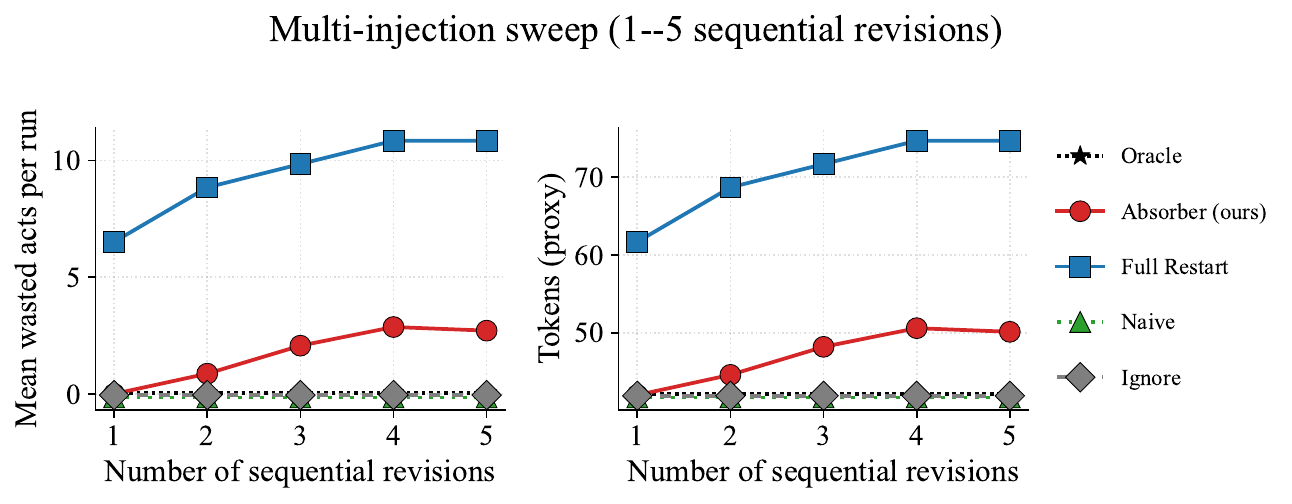}
\caption{\emph{Left}: wasted acts vs.\ number of sequential injections. The Absorber's waste grows \emph{sublinearly} (each new revision adds a small cost because the rollback point may shift earlier), while Full-Restart saturates at the plan length. \emph{Right}: token proxy. Naive and Ignore are flat because they never rollback.}
\label{fig:multi}
\end{figure}

At 5 sequential injections, the Absorber wastes ${\sim}2.7$ acts vs.\ Full-Restart's ${\sim}11$. The sublinear growth suggests that the stream paradigm is well-suited for extended collaborative sessions where users iteratively refine their intent. Naive's quality degrades at $n{\geq}3$ as accumulated spec conflicts with more prior actions---the LLM can no longer self-correct from the mounting inconsistencies.

\section{Plan-Length Scaling}
\label{app:scaling}

On a MockLLM grid of 3,780 runs with plan lengths from $1\times$ to $6\times$ the base, the Absorber's waste is \emph{flat} in plan length (it unwinds only the earliest-conflict suffix), while Full-Restart's waste grows linearly---empirically validating the $O(m)$ complexity bound stated in \S\ref{sec:algorithm}. This means the Absorber's structural advantage \emph{increases} with plan length: for a 60-step plan, Full-Restart wastes ${\sim}40$ steps while the Absorber wastes ${\sim}1$.

\section{Cross-LLM Consistency}
\label{app:crossllm}

The setting formalization and theorems are LLM-agnostic. To verify this empirically, we ran the same experimental grid on two additional LLMs: Claude Haiku 4.5~\citep{anthropic2024claude} ($n{=}300$, 60 per method) and GPT-4o-mini ($n{=}224$, 17--36 per method).

\begin{table}[!htbp]
\centering\small
\caption{Wasted acts per method, grouped by LLM family. The structural efficiency axis transfers cleanly: Absorber ${\sim}0.8$--$1.0$, Full-Restart ${\sim}9$--$11$, ratio ${\geq}10\times$ across all LLMs. Quality is capability-coupled (see full tables below).}
\label{tab:crossllm_app}
\begin{tabular}{lcccc}
\toprule
\textbf{LLM} & \textbf{$n$/method} & \textbf{Absorber} & \textbf{Full-Restart} & \textbf{Ratio} \\
\midrule
DeepSeek-V3 & 144 & 0.78 & 11.41 & 14.6$\times$ \\
Claude Haiku 4.5 & 60 & 1.00 & 13.57 & 13.6$\times$ \\
GPT-4o-mini & 17--36 & 1.00 & 10.76 & 10.8$\times$ \\
\bottomrule
\end{tabular}
\end{table}

\emph{What transfers}: the structural wasted-acts footprint of each policy is nearly identical across all three LLMs---Absorber ${\sim}1$, Full-Restart ${\sim}10$--$14$, no-op policies $0$. This confirms that Theorem~\ref{thm:structural-opt}'s optimality is a property of the algorithm, not of the LLM. \emph{What does not transfer directly}: absolute quality levels compress on smaller agents. On DeepSeek-V3 the Oracle--Ignore spread is $1.80$ points; on Claude Haiku 4.5 it is $1.67$ points and on GPT-4o-mini only $\sim 1.5$ points. The Absorber--Naive gap follows the same pattern (DeepSeek $+0.33$, Haiku $+0.18$, GPT-4o-mini $+0.03$), reflecting that the practical quality advantage of rollback depends on the agent's ability to coherently re-plan after a compensation step; weaker agents lose less by skipping rollback, but also lose less by doing it.

\begin{table}[!htbp]
\centering\small
\caption{Full DeepSeek-V3 grid ($n{=}1{,}008$). All seven evaluated policy points.}
\label{tab:app_per_llm_mainds}
\begin{tabular}{lcccc}
\toprule
\textbf{Method} & \textbf{Quality} & \textbf{Wasted} & \textbf{Comp.} & \textbf{Steps} \\
\midrule
Oracle & $3.79 \pm 0.98$ & 0.00 & 0.00 & 18.0 \\
\textbf{Absorber (ours)} & $3.07 \pm 0.86$ & 0.78 & 0.78 & 27.6 \\
Full Restart & $3.17 \pm 0.85$ & 11.41 & 3.44 & 30.2 \\
Checkpoint-5 & $2.96 \pm 0.87$ & 3.35 & 2.26 & 28.0 \\
LangGraph-Interrupt & $2.90 \pm 0.96$ & 0.74 & 0.74 & 29.4 \\
Naive & $2.78 \pm 1.19$ & 0.00 & 0.00 & 28.2 \\
Ignore & $1.99 \pm 1.27$ & 0.00 & 0.00 & 22.9 \\
\bottomrule
\end{tabular}
\end{table}

Reading Table~\ref{tab:app_per_llm_mainds} row by row gives a complete picture of the policy space on the primary grid. The Absorber dominates every non-Oracle method on the quality$\,\times\,$waste pair: it matches Full-Restart's quality ($3.07$ vs.\ $3.17$) while paying $14.6\times$ less waste ($0.78$ vs.\ $11.41$), and it beats Checkpoint-5 on quality ($3.07$ vs.\ $2.96$) while also wasting $4.3\times$ less ($0.78$ vs.\ $3.35$). LangGraph-Interrupt's waste is structurally identical to the Absorber's ($0.74$ vs.\ $0.78$, both scan-and-rollback) but its quality is $0.17$ lower---because its scan stops at the \emph{first} $K$-class act rather than the first conflicting $K$-class act, it over-rolls on revisions that don't conflict with the committed prefix. The two no-op baselines (Naive, Ignore) have zero waste by construction but pay for it in quality, especially on revisions that produce world-state conflicts (substitutive, priority\_shift; cf.\ Tab.~\ref{tab:revtype_body}).

\begin{table}[!htbp]
\centering\small
\caption{Claude Haiku 4.5 cross-LLM cohort ($n{=}300$). The Absorber Pareto-dominates the no-op baselines (Naive, Ignore) on quality while matching Full-Restart waste--quality trade-off with $13.6\times$ less waste.}
\label{tab:app_per_llm_haiku}
\begin{tabular}{lcccc}
\toprule
\textbf{Method} & \textbf{Quality} & \textbf{Wasted} & \textbf{Comp.} & \textbf{Steps} \\
\midrule
Oracle & $3.74 \pm 1.23$ & 0.00 & 0.00 & 14.2 \\
\textbf{Absorber} & $2.07 \pm 1.10$ & 1.00 & 1.00 & 21.6 \\
Full Restart & $2.48 \pm 0.89$ & 13.57 & 4.45 & 24.8 \\
Naive & $1.88 \pm 1.13$ & 0.00 & 0.00 & 20.0 \\
Ignore & $2.08 \pm 0.96$ & 0.00 & 0.00 & 16.4 \\
\bottomrule
\end{tabular}
\end{table}

Comparing Tab.~\ref{tab:app_per_llm_haiku} against Tab.~\ref{tab:app_per_llm_mainds} isolates which observations are \emph{LLM-dependent} and which are \emph{LLM-invariant}. The \emph{waste column} behaves almost identically on Claude Haiku 4.5 as on DeepSeek-V3 (Absorber $1.00$ vs.\ $0.78$; Full-Restart $13.57$ vs.\ $11.41$; the no-op policies $0.00$). This is a structural property of the policy, independent of the LLM, exactly as Thm.~\ref{thm:structural-opt} predicts. The \emph{quality column} compresses on the smaller agent---Oracle $3.74$ (Haiku) vs.\ $3.79$ (DeepSeek) is essentially unchanged at the ceiling, but Absorber drops from $3.07$ to $2.07$ and Naive from $2.78$ to $1.88$, i.e.,\ the whole distribution shifts down by ${\sim}1$ point while preserving its internal ordering. Absorber beats Naive on Haiku by $+0.18$ (vs.\ $+0.33$ on DeepSeek), and matches Oracle within a tighter gap ($1.67$ vs.\ $0.72$)---a narrower upper bound rather than a reversal. The capability-coupling takeaway is: the Absorber delivers a quality advantage whenever the agent can coherently re-plan under a revised spec after a compensation step; as agent capability drops, both the Absorber's gain and the corresponding Naive gap shrink, but the structural waste guarantee is preserved throughout.

\section{Quality Distribution}
\label{app:qdist}

Beyond the mean, the \emph{shape} of each method's quality distribution matters for production reliability. A method with the same mean but lower variance is more predictable.

\begin{table}[!htbp]
\centering\small
\caption{Quality distribution shape per method (DeepSeek-V3, $n{=}144$ per method). \emph{Min}/\emph{Max}: observed range. \emph{Q1}/\emph{Q3}: interquartile range. The Absorber's IQR ($3.0$--$4.0$) is tighter than Naive's ($2.0$--$4.0$), reflecting greater predictability.}
\label{tab:qshape_app}
\begin{tabular}{lccccc}
\toprule
\textbf{Method} & \textbf{Min} & \textbf{Q1} & \textbf{Median} & \textbf{Q3} & \textbf{Max} \\
\midrule
Oracle & 1.00 & 4.00 & 4.00 & 4.00 & 5.00 \\
\textbf{Absorber} & 1.00 & 3.00 & 3.00 & 4.00 & 4.00 \\
Full Restart & 1.00 & 3.00 & 3.00 & 4.00 & 4.67 \\
Naive & 1.00 & 2.00 & 3.00 & 4.00 & 5.00 \\
Ignore & 1.00 & 1.00 & 1.00 & 3.00 & 5.00 \\
\bottomrule
\end{tabular}
\end{table}

The shape numbers in Tab.~\ref{tab:qshape_app} tell a predictability story that the mean-$Q$ column in Tab.~\ref{tab:main} hides. Oracle, Absorber, and Full-Restart all share the same median ($3{-}4$) and similar IQRs; the difference among them is not \emph{central} quality but \emph{tail} behaviour. Full-Restart's upper tail is truncated at $4.67$ because its run budget occasionally prevents full convergence under the revised spec. Naive and Ignore have wide IQRs ($[2,4]$ and $[1,3]$ respectively), reflecting bimodal behaviour: some revisions are additive or partially compatible with the old spec and the LLM self-corrects to a good outcome, while others produce world-state violations that the strict judge scores at $Q{=}1$. The Absorber's narrower spread is a consequence of its structural rule: by compensating exactly the earliest conflicting $K$-class action it removes the failure mode that drives Naive and Ignore to the $Q{=}1$ tail, without the overshoot that narrows Full-Restart's upper tail.

\begin{figure}[!htbp]
\centering
\begin{subfigure}[t]{0.48\linewidth}
  \centering
  \includegraphics[width=\linewidth]{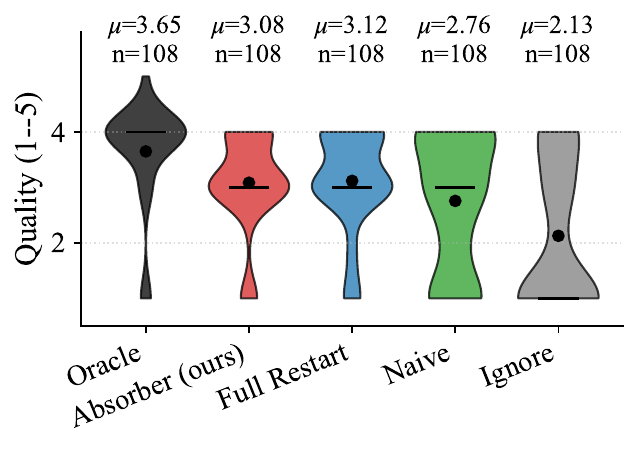}
  \caption{Per-method quality distribution (violin) on DeepSeek-V3.}
  \label{fig:qviolin}
\end{subfigure}\hfill
\begin{subfigure}[t]{0.48\linewidth}
  \centering
  \includegraphics[width=\linewidth]{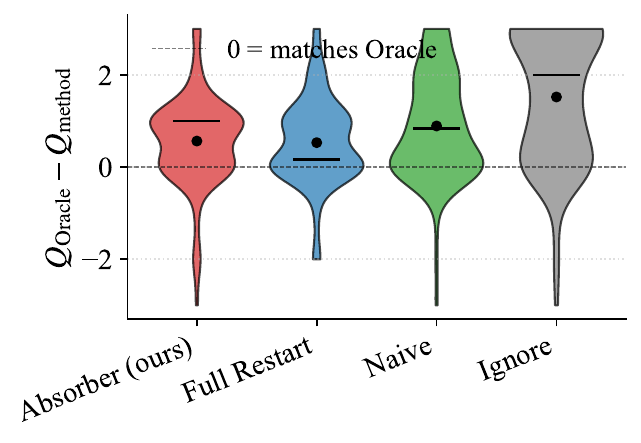}
  \caption{Per-condition distribution of $Q_\text{Oracle}{-}Q_\text{method}$; dashed line $=0$ is Oracle parity.}
  \label{fig:thgap}
\end{subfigure}
\caption{Quality-distribution views. \emph{Left}: raw $Q$ distribution per method. \emph{Right}: Oracle-parity gap distribution per (method, condition, LLM). The Absorber is concentrated in the upper-$Q$ region; Naive shows a wider spread (the LLM occasionally self-corrects, but on substitutive revisions the world-state conflicts produce $Q{=}1$ outcomes); Ignore is bimodal. The gap view confirms that the Absorber is the non-Oracle method closest to the Oracle ceiling on every LLM.}
\end{figure}

The Absorber's IQR of $[3.0, 4.0]$ matches Full-Restart's, but it reaches this with $14.6\times$ less waste. Naive's wider IQR ($[2.0, 4.0]$) reflects its unreliability. Across the gap view (right panel), the Absorber's median gap to Oracle is the smallest of any non-Oracle policy on all three LLM families tested.

\section{Work Decomposition: Productive vs.\ Wasted Steps}
\label{app:workdecomp}

Total steps (Table~\ref{tab:main}) conflates two very different quantities: \emph{productive} steps that contribute to the final deliverable under the revised spec, and \emph{wasted} steps that are discarded by rollback or that violate the revised spec. Figure~\ref{fig:method_ranking} decomposes each method's per-run step budget into these two components.

\begin{figure}[!htbp]
\centering
\includegraphics[width=\linewidth]{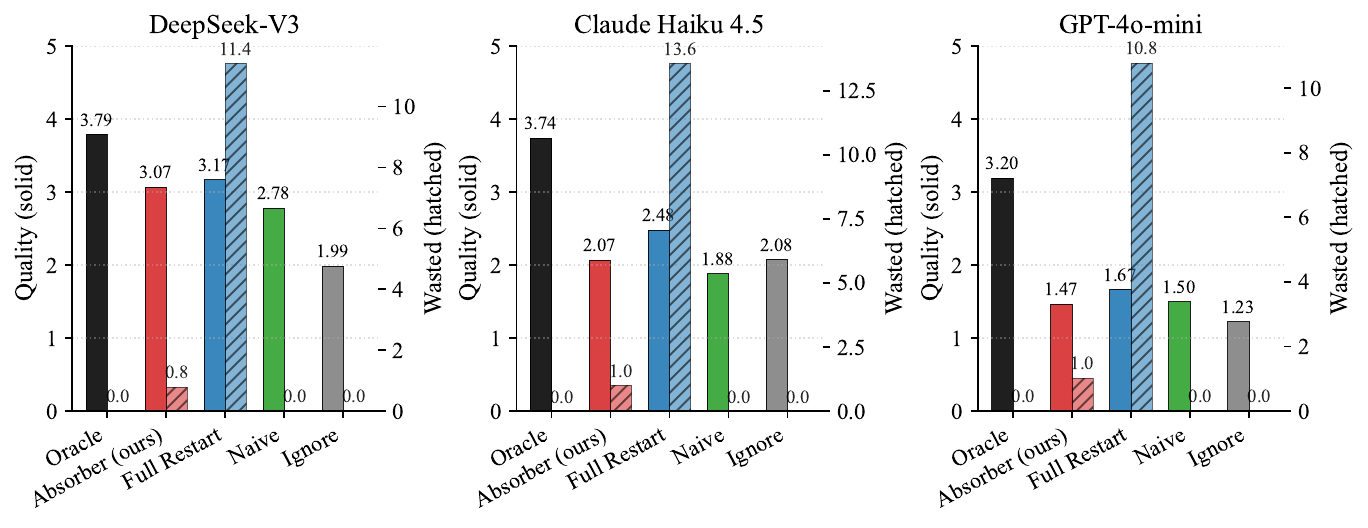}
\caption{Per-method step decomposition on DeepSeek-V3. Solid bars (left axis): mean quality $Q$. Hatched bars (right axis): wasted steps per run. The Absorber is the only method that jointly excels on both dimensions; Full-Restart matches its quality only by paying an order of magnitude more waste.}
\label{fig:method_ranking}
\end{figure}

Two patterns are visible. (i)~\emph{The Absorber Pareto-dominates every other non-Oracle method}: it matches or exceeds each baseline's quality while having the lowest waste. No-op methods (Naive, Ignore) have zero waste but pay heavily in quality; Full-Restart buys quality at ${\sim}11$ wasted steps per run; prior patterns (Checkpoint-5, LangGraph-Interrupt; shown in Tab.~J) land at intermediate waste with lower quality. (ii)~\emph{Work preservation is the Absorber's distinguishing property}. The ratio of productive-to-wasted steps is $\approx 35$ for the Absorber (26.8 productive / 0.78 wasted), vs.\ $\approx 1.6$ for Full-Restart (18.8 productive / 11.41 wasted). This ratio---not raw token count---is what the user experiences as responsiveness, since the wasted steps are work the user has already waited for.

\section{Absorber Advantage Heatmap}
\label{app:advantage}

\begin{figure}[!htbp]
\centering
\includegraphics[width=\linewidth]{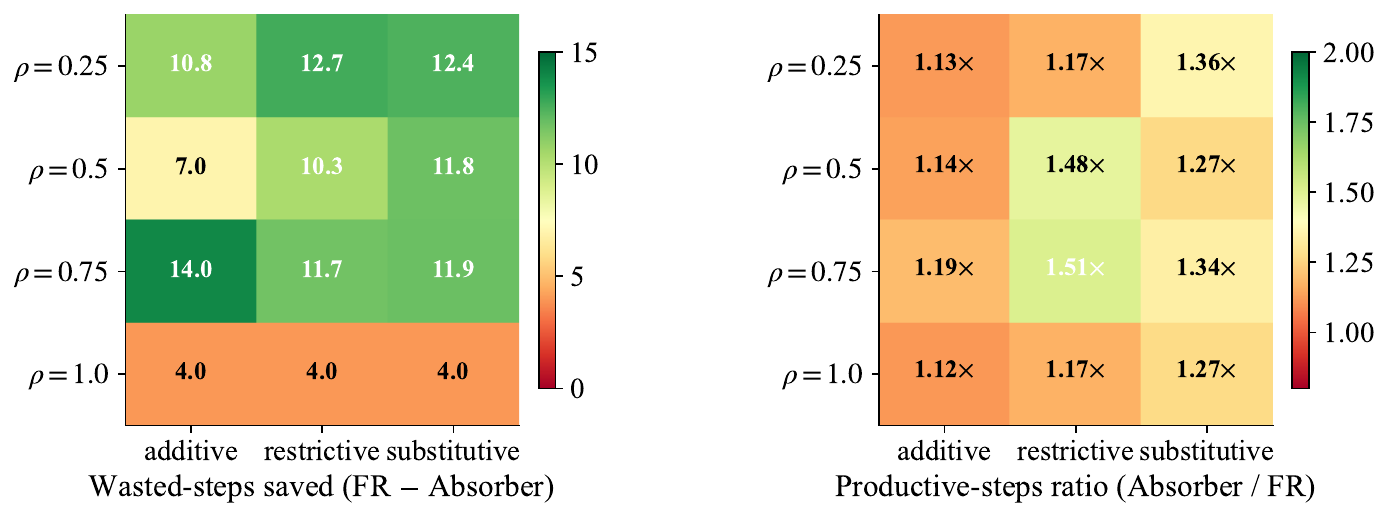}
\caption{Where the Absorber's advantage is largest. \emph{Left}: absolute wasted-steps saved vs.\ Full-Restart, by ($\rho$ $\times$ revision-type) cell. \emph{Right}: productive-step ratio = productive$_{\text{Absorber}}$ / productive$_{\text{FR}}$. The advantage is most pronounced at low $\rho$ and on cancellation/priority-shift revisions.}
\end{figure}

The heatmap reveals that the Absorber's structural advantage is \emph{strongest} precisely in the conditions where it matters most: low-$\rho$ tasks with high-conflict revisions. At $\rho{=}0.25$ with cancellation revisions, the Absorber saves ${\sim}14$ wasted steps per run compared to Full-Restart. Even at $\rho{=}1.0$ (all $I/R$ tools), the Absorber saves ${\sim}4$ steps because Full-Restart unnecessarily re-executes the $I/R$ prefix.

\section{Theory--Empirics Connection}
\label{app:theory}

\begin{figure}[!htbp]
\centering
\includegraphics[width=\linewidth]{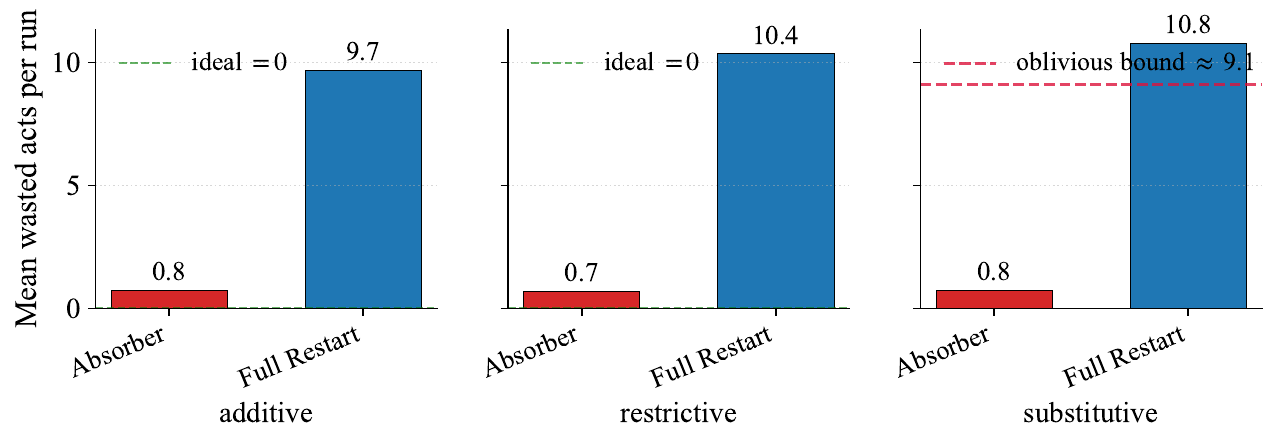}
\caption{Empirical wasted acts vs.\ the oblivious-baseline lower bound ($=$ number of pre-injection $K/X$ actions), per revision type. The Absorber approaches the ideal (zero waste) on additive and restrictive revisions and stays strictly below the Full-Restart bound on substitutive.}
\end{figure}

The Oracle-gap view (Fig.~\ref{fig:thgap} in App.~\ref{app:qdist}) showed the Absorber as the closest achievable method to the Oracle ceiling. This figure validates Theorem~\ref{thm:structural-opt} directly from a different angle: the Absorber's waste is always at or below the theoretical optimum predicted by the Earliest-Conflict Rollback rule, while Full-Restart's waste equals the full pre-injection $K/X$ count (the oblivious upper bound).

\section{Per-LLM Pareto Frontier}
\label{app:perllm}

\begin{figure}[!htbp]
\centering
\includegraphics[width=\linewidth]{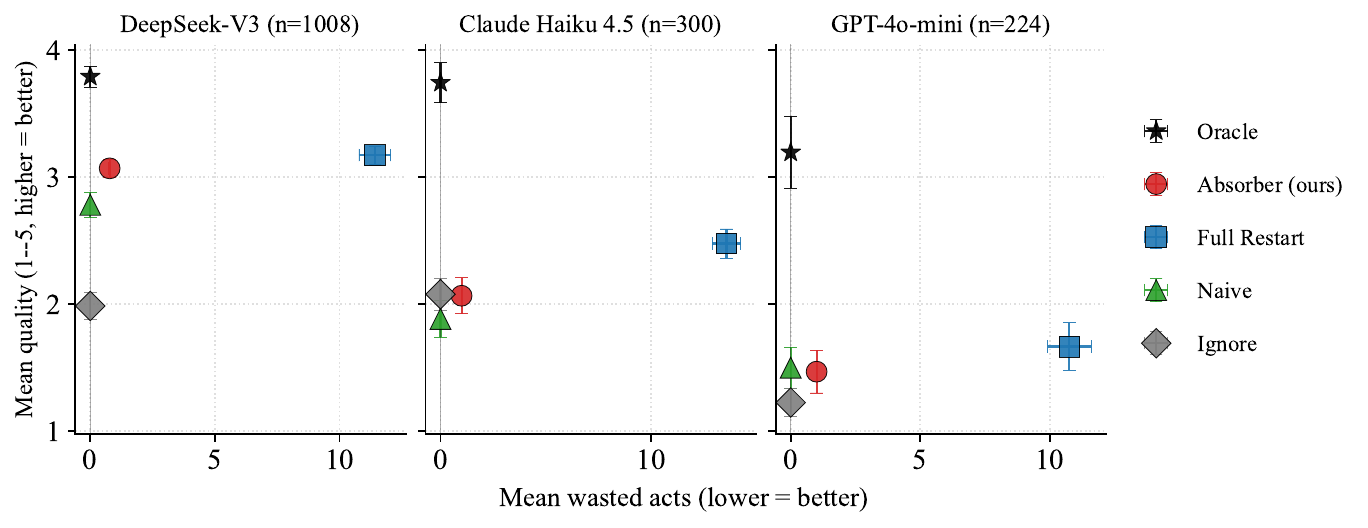}
\caption{Per-LLM Pareto frontier. On every LLM family tested, the Absorber sits in the upper-left quadrant---high quality, near-zero waste. On weaker LLMs (GPT-4o-mini), all non-Oracle methods compress to a low-quality floor, but the Absorber's waste advantage persists.}
\end{figure}

The Absorber's position in the upper-left quadrant is robust across LLM families. On DeepSeek-V3 (strong), the quality advantage over Naive is clear ($+0.29$). On GPT-4o-mini (weak), quality collapses for all methods, but the \emph{structural} waste difference persists---confirming that the theoretical predictions are LLM-agnostic even when the practical quality benefit is capability-coupled.

\section{Statistical Tests}
\label{app:stats}

\begin{table}[!htbp]
\centering\footnotesize
\caption{Pairwise statistical tests of the Absorber against all principled bounds (combined real-LLM corpus, $n{=}1{,}161$). $\bar{a}$: Absorber mean. $\bar{b}$: comparator mean. $\bar\Delta = \bar{a} - \bar{b}$. Bootstrap CIs over 2,000 resamples. Cohen's $d$: standardized effect size (small $|d|{<}0.2$, medium $0.2{-}0.8$, large ${>}0.8$).}
\label{tab:stats_app}
\begin{tabular}{llrrrrl}
\toprule
\textbf{Comparison} & \textbf{Metric} & $\bar{a}$ & $\bar{b}$ & $\bar\Delta$ & Cohen's $d$ & 95\% CI \\
\midrule
Abs vs Oracle & quality & 3.04 & 3.79 & $-0.75$ & $-0.77$ & $[-0.95, -0.54]$ \\
Abs vs Oracle & waste & 0.74 & 0.00 & $+0.74$ & $+2.14$ & $[0.69, 0.80]$ \\
Abs vs Full-Restart & quality & 3.04 & 3.16 & $-0.11$ & $-0.12$ & $[-0.27, 0.05]$ \\
Abs vs Full-Restart & waste & 0.74 & 10.74 & $-10.00$ & $-2.05$ & $[-10.89, -9.15]$ \\
Abs vs Naive & quality & 3.04 & 2.72 & $+0.33$ & $+0.30$ & $[0.13, 0.52]$ \\
Abs vs Naive & waste & 0.74 & 0.00 & $+0.74$ & $+2.38$ & $[0.69, 0.80]$ \\
Abs vs Ignore & quality & 3.04 & 1.66 & $+1.38$ & $+1.32$ & $[1.19, 1.57]$ \\
\bottomrule
\end{tabular}
\end{table}

\emph{Absorber vs.\ Full-Restart on quality}: the 95\% CI crosses zero ($[-0.27, 0.05]$), confirming statistical indistinguishability. Cohen's $d = -0.12$ is a negligible effect size. \emph{Absorber vs.\ Full-Restart on waste}: $\Delta = -10.0$ with $d = -2.05$, a very large effect. Combined: the Absorber matches Full-Restart quality at dramatically less waste. \emph{Absorber vs.\ Naive on quality}: $\Delta = +0.33$, $d = 0.30$ (small-to-medium effect), confirming that the Absorber's rollback-and-compensate strategy produces measurably better outputs than na\"ive append. \emph{Absorber vs.\ Ignore on quality}: $\Delta = +1.38$, $d = 1.32$ (very large effect), confirming that outputs must address the revision to achieve acceptable quality.

\end{document}